\documentclass[conference]{IEEEtran}
\IEEEoverridecommandlockouts
\usepackage{graphicx}
\usepackage{epstopdf}
\usepackage{epsfig}
\usepackage{xcolor}
\usepackage{float}
\usepackage{color,soul}
\usepackage{amsmath}
\usepackage{amsthm}
\usepackage{amssymb}
\usepackage{bbm}
\usepackage{array}
\usepackage{algorithm,algpseudocode}

\newcommand{\myb}[1]{\boldsymbol{#1}}
\newtheorem{prop}{Proposition}

\usepackage{cite}
\usepackage{amsmath,amssymb,amsfonts}
\usepackage{graphicx}
\usepackage{textcomp}
\usepackage{xcolor}
\usepackage{breqn}

\def\BibTeX{{\rm B\kern-.05em{\sc i\kern-.025em b}\kern-.08em
    T\kern-.1667em\lower.7ex\hbox{E}\kern-.125emX}}
    
\begin{document}

\title{MaxModShift: Model Privacy via Designed Shifts
\thanks{This work has been funded in part by one or more of the following grants: ONR N00014-22-1-2363, NSF CCF 2148313, ARO W911NF2410094, NSF CCF-2311653 and is supported in part by funds from federal agency and industry partners as specified in the Resilient \& Intelligent NextG Systems (RINGS) program, ARO W911NF1910269, and NSF Center for Pandemic Insights, DBI-2412522.}
}

\author{\IEEEauthorblockN{Nomaan A. Kherani}
\IEEEauthorblockA{
\textit{University of Southern California}\\
kherani@usc.edu}
\and
\IEEEauthorblockN{Urbashi Mitra}
\IEEEauthorblockA{
\textit{University of Southern California}\\
ubli@usc.edu} }

\maketitle

\begin{abstract}
Model learning by an eavesdropper is treated as an estimation problem in a federated environment. The Fisher Information Matrix for the eavesdropper's estimation problem is driven to singularity through a signaling design; this ensures that the eavesdropper cannot learn the model. {Herein, the innovation of prior designs is that model shifts
 are designed to maximize the difference in the model learned by Eve and the central server while satisfying a transmission power constraint for the agents.}  Two shift schemes are provided.  
 {MaxModShift outperforms a prior ModShift design while requiring lesser transmission power.} Compared to a noise injection scheme, MaxModShift performs better while requiring a lower bandwidth secret channel and a reduced average power consumption. 
\end{abstract}


\section{Introduction}

In federated learning (see {\em e.g.} \cite{kairouz2021advances}), agents provide local model information to a global server while maintaining privacy of the local data resident at each agent. However, inferences about an agent's data can still be made by exploiting the shared local model updates, motivating the development of additional strategies to ensure data privacy \cite{privacyreview,secagg,fastsecagg,dp,tasnim2023approximating}; however these methods do not ensure privacy of the global model.

Recent work has begun to address the issue of {\bf model privacy}\cite{modelprivacy1,modelprivacy2,modelprivacy3}. In \cite{modelprivacy1}, the model is protected from the participating agents, but not eavesdroppers. In contrast, \cite{modelprivacy2} protects the model from eavesdroppers, but does not enable agent model learning.  While \cite{modelprivacy3} protects model privacy from eavesdroppers, a very constrained scenario is considered: only a link between a single agent the global server (amongst many) is compromised. 
%In this case, algorithms in which agents share model increments are inherently protected.  
Herein, we consider the problem of model privacy when {\bf all} uplinks between the agents and the global server are eavesdropped.

Our approach is inspired by the creation of statistically hard estimation problems for the eavesdropper through signal shaping \cite{li2024channel,li2024channelb, li2024optimized}. %In \cite{da2024guaranteed}, wireless communications are made private through randomizing the {\em structure} of the modulation.  
In \cite{li2024channel, li2024channelb, li2024optimized}, localization is made private by modifying the channel perceived by the eavesdropper.  The Fisher Information Matrix of the estimation problem undertaken by the eavesdropper is driven to singularity through transmitted signal precoding.  Such an approach is undertaken herein for the new problem of model privacy in federated learning or distributed optimization.  As in \cite{li2024channel,li2024channelb, li2024optimized}, only a modest amount of information is shared with the global server by the agents in order to achieve model privacy.

 \begin{figure}[t]
 
 \begin{center}
\includegraphics[width=0.5\textwidth]{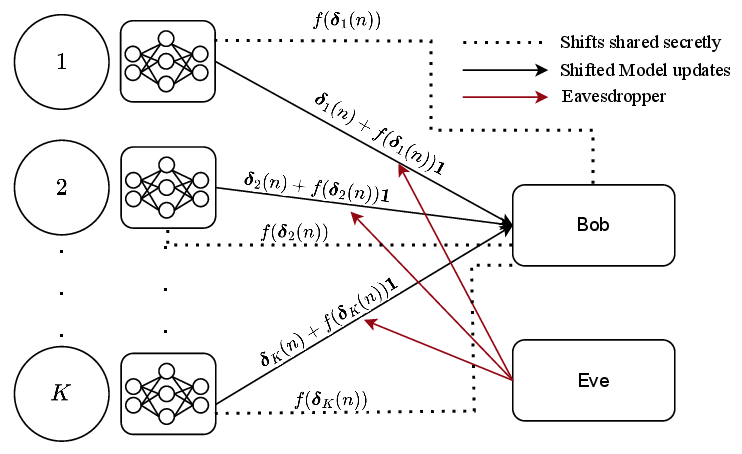}

\caption{System model for distributed optimization, model shift and eavesdropping.}

\label{fig:FL}
\end{center}

\end{figure}
{In our prior design, ModShift \cite{modshift}, an intentional model shift is introduced at each agent before transmission to the central server; however, the proposed designs were not optimized.} Model shift due to data heterogeneity has proven to be detrimental for the learning of global models \cite{shift1,shift2}. {In this work we design shifts that both induce statistical hardness for the eavesdropper as well as maximize the model shift under a power constraint.} {The proof techniques in \cite{modshift} and \cite{jsac} focus on FIM computation while in this paper, the proofs involve solving an optimization problem via the method of Lagrange multipliers.} As in \cite{modshift},  we use FedAvg \cite{mcmahan2017} as our exemplar learning scheme; however, our approach can be adapted to other federated or distributed optimization methods (see {\em e.g.} \cite{verma2023maximal}). A block version of ModShift was introduced in \cite{jsac}.

The main contributions of this paper are:
\begin{enumerate}
\item We re-derive the Fisher Information for shift learning problem for an eavesdropper for a more general shift model than that proposed in \cite{modshift}.  We further provide the conditions needed for singularity.
\item We provide a family of shift designs that achieve the desired singularity and prove properties of these designs.  %Although not explicitly considered herein, we can show that the new designs pass a test that can be applied at the eavesdropper to detect if model updates have been tampered with.
\item The new shift designs, which satisfy a power constraint, are designed to maximize the model shift, which correlates with a higher loss function for the eavesdropper.
 \item Several theoretical results are provided including the provable goodness of the method.
%\item  There is an inherent decision problem for our proposed strategy to select between two shift designs.  
%A regret bound is computed for this decision making problem.
    \item The proposed method is compared to ModShift \cite{modshift} and noise injection methods.  MaxModShift leads Eve to a worse performance than ModShift while requiring $24\%$ of the transmission power. Noise injection fails the tamper test and requires a greater bandwidth secret channel and further results in a significantly lower estimation loss for Eve than MaxModShift. 
\end{enumerate}
There are a number of technical results that are not included herein due to space limitations, but can be shown. In particular, the new schemes (MaxModShift) are shown to pass a convergence test at the eavesdropper and thus the eavesdropper cannot detect the model shifts. The algorithm for implementing the model shifts can also be proven to converge, following the approaches in \cite{modshift}.

\section{System Model}
\label{sec-sys}
Our system model from \cite{modshift} is reviewed.
 We consider a network consisting of $K$ agents that communicate with a central server (Bob) to collectively learn a global model $\boldsymbol{w}^{*} \in \mathbb{R}^{d}$. An unauthorized receiver (Eve) tries to learn $\myb{w}^{*}$ by eavesdropping on the uplink communication between the agents and the server.  Each agent has a local dataset $\myb{D}_{k}$ comprised of realizations of feature vectors, $\boldsymbol{x}_{k,i} \in \mathbb{R}^{d}$, and their corresponding labels $y_{k,i}$, where 
 $\myb{D}_{k} = \left [[\boldsymbol{x}_{k,1}^{T},y_{k,1}]^{T} , \cdots, [\boldsymbol{x}_{k,m_k}^{T},y_{k,m_{k}}]^{T} \right]$.  The global data set is thus,
$\myb{D} =  \left[\myb{D}_{1}, \cdots ,\myb{D}_{K}
\right] \in  \mathbb{R}^{(d+1) \times m},$ where $m = \sum_{k=1}^{K}m_{k}.$ 

 The function $l(\boldsymbol{w},\boldsymbol{x}_{k,i},y_{k,i})$ represents the loss of the $i$th data sample of agent $k$.
 Our aim is to find the $\boldsymbol{w}^{*}$ that minimizes the global loss function at the server
 \begin{align}
    \label{loss}
     F(\boldsymbol{w}) = \frac{1}{m}\sum_{k=1}^{K}\sum_{i=1}^{m_{k}}l(\boldsymbol{w},\boldsymbol{x}_{k,i},y_{k,i}).
 \end{align}
 The gradient computed by agent $k$ on dataset $\myb{D}_{k}$ with a weight $\myb{w}$ is given by,
\begin{equation}
    \myb{g}_{k}(\myb{w},\myb{D}_{k}) = \sum_{i=1}^{m_{k}} \frac{\nabla l(\myb{w},\myb{x}_{k,i},y_{k,i})}{m_{k}}.    
\end{equation}
We make no assumptions about the strategy by which the gradients are generated, thus our work applies to both neural network architectures as well model-based approaches.
We use FedAvg \cite{mcmahan2017} with a slight modification to find the parameter $w^{*}$ as presented in Algorithm (\ref{algo}). In iteration $n$, after performing $R$ steps of local gradient descent with weight $\myb{w}(n)$, the devices share the \emph {difference} ${\myb{\delta}_{k}(n)} = \myb{w}_{k,R}(n)-\myb{w}(n)$ between their resulting local model $\myb{w}_{k,R}(n)$ and the global model $\myb{w}(n)$ versus sharing $\myb{w}_{k,R}(n)$. This formulation introduces an initial model shift due to random initialization and necessitates that Eve eavesdrop in every communication round to attempt to reconstruct the model trajectory. 
\begin{algorithm}[h]
\caption{The $K$ clients are indexed by $k$; $R$ is the number of local epochs and $\eta$ is the learning rate.}
\begin{algorithmic}[h]
\State \textbf{Server executes:}
\State Initialize $\myb{w}(0)$
\For{each round $n = 0, 1, \dots N-1$}
    \For{each client $k$ \textbf{in parallel}}
        \State $\myb{w}_{k,0}(n) \gets \myb{w}(n)$
        \For{each local epoch $r$ from $0$ to $R-1$}
            \State $\myb{w}_{k,r+1}(n) = \myb{w}_{k,r}(n) - \eta g_{k}(\myb{w}_{k,r}(n),\myb{D}_{k})$
        \EndFor
        \State $\myb{\delta}_{k}(n) \gets \myb{w}_{k,R}(n)-\myb{w}(n)$
    \EndFor
    \State $\myb{w}(n+1) \gets \myb{w}(n) + \sum_{k=1}^{K}\frac{m_{k}}{m}\myb{\delta}_{k}(n)$
\EndFor
\end{algorithmic}
\label{algo}
\end{algorithm}

We assume that each agent uses an orthogonal channel for transmission of $\myb{\delta}_{k}(n)$ to Bob, such as orthogonal frequency-division multiplexing (OFDM) \cite{ofdmbook}. Under the assumption of white Gaussian noise across the sub-channels and flat fading with known channel state information for both Bob and Eve, we have that the effective received signals per orthogonal channel in iteration $n$ for each receiver is
\begin{eqnarray}
    \label{Bob1}
    {\underline{\myb{y}}_{k}^{u}(n)} & = &\myb{\delta}_{k}(n) + \underline{\myb{z}}_{k}^{u}(n),
\end{eqnarray}
where $\underline{\myb{z}}_{k}^{u}(n) \sim \mathcal{CN}(0,\frac{\sigma_{u}^{2}}{h_{k}^{u}(n)^{2}}\myb{I})$ is the distribution of the noise conditioned on the known channel state $h_{k}^{u}(n)$ and $u \in \{ B,E\}$. Both Bob and Eve wish to estimate the global model $w^{*}$ from their received signals.

In the sequel, we modify the transmitted signals, sharing this change with Bob over a secure channel \footnote{Note that communication privacy for a modest amount of shared information can be achieved via the methods in \cite{da2024guaranteed,li2024channel,li2024channelb, li2024optimized}.}, to mislead Eve to the incorrect model.

\section{Proposed Scheme: ModShift}
\label{modshiftprop}
Inspired by ModShift \cite{modshift}, we propose to introduce shifts to $\myb{\delta}_{k}(n)$ to degrade the {estimation and therefore learning performance} of the algorithm at Eve. With the addition of the shift, Eve estimates $\myb{\delta}_{k}(n)$ from 
\begin{align}
    \label{Eve}
    \underline{\myb{y}}_{k}^{E}(n) &  = \myb{\delta}_{k}(n) + \myb{\gamma}_{k}(n)^{T}\myb{\delta}_{k}(n)\myb{u} + \underline{\myb{z}}_{k}^{E}(n), \\
    & = \myb{S}_{k}(n)\myb{\delta}_{k}(n) + \underline{\myb{z}}_{k}^{E}(n).
\end{align}
where $\myb{S}_{k}(n) = \myb{I} + \myb{u}\myb{\gamma}_{k}(n)^{T}$,$\myb{\gamma}_{k}(n) \in R^{d}$ is a vector which the agents design, $\myb{u}$ is the shift direction vector and $\myb{\gamma}_{k}(n)^{T}\myb{\delta}_{k}(n)$ is the shift introduced by the agents.
Thus, Eve's update is 
\begin{align}
    \label{update}
    \myb{w}(n+1)^{E} = \myb{w}(n)^{E} + \sum_{k=1}^{K}\frac{m_{k}}{m}\underline{\myb{y}}_{k}^{E}(n)
\end{align}
In the sequel, we provide a design for $\myb{\gamma}_{k}(n)$ that yields a challenging estimation problem for Eve.  
After compensating for shifts, Bob's received signal remains as in Equation \eqref{Bob1}.
The \emph{model shift} experienced by Eve is given by $\lim_{n\rightarrow \infty} \myb{w}(n)^{E} - \myb{w}$; our goal is that this converges to a non-zero quantity.
%The goal is for Eve's estimate to converge to the incorrect model.  The difference between the true and estimated model is called the \emph{model shift}.
As in \cite{modshift},
we minimize the use of the secret channel by characterizing the shift by a scalar $\myb{\gamma}_{k}(n)^{T}\myb{\delta}_{k}(n)$ and $\myb{u}$ is agreed upon prior to the learning as in ModShift \cite{modshift}. Note that in our FIM computation, we assume Eve knows $\myb{u}$.
%The shift is a function of the gradient to be transmitted and is given by $f_{k,n}(\myb{\delta}_{k}(n))\mathbbm{1}$.
\section{Shift Design}
\label{shiftdes}
%As our shift design is generalized over that in \cite{modshift}, we need to recompute the FIM  for a general shift direction $\myb{u}$ and determine the conditions on $\myb{\gamma}_{k}(n)$ for the FIM of $\myb{\delta}_{k}(n)$ to be singular \footnote{Part of Proposition 1 is present in \cite{kherani_AI4G}. \um{this is vague}}.
We first review our result from \cite{kherani_AI4G} which determines the conditions on $\myb{\gamma}_{k}(n)$ for the FIM of $\myb{\delta}_{k}(n)$ to be singular.
\begin{prop}
\label{fimblock}
    The FIM $\myb{J}(\myb{\delta}_{k}(n))$ of $\myb{\delta}_{k}(n)$ for the signal model in Equation \eqref{Eve} is given by
    \begin{align}
        \myb{J}(\myb{\delta}_{k}(n)) & = \frac{2h_{k}^{E}(n)^{2}}{\sigma_{E}^{2}}\myb{S}_{k}(n)^{T}\myb{S}_{k}(n).
    \end{align}
    $\myb{J}(\myb{\delta}_{k}(n))$ is singular when :
        $\myb{\gamma}_{k}(n)^{T}\myb{u} = -1.
$   
\end{prop} 
Proposition \ref{fimblock} does not consider $\myb{\gamma}_{k}(n)$ to be a function of $\myb{\delta}_{k}(n)$. However, allowing $\myb{\gamma}_{k}(n)$ to be a function of $\myb{\delta}_{k}(n)$  will enable improved designs. 
%\no{In the workshop submission, we have found conditions under which the block FIM is singular. Setting the block size to 1 will give us the result from the current proposition. In the ModShift paper, $\myb{u}$ was equal to $\mathbbm{1}$ and so is a special case of the above proposition}
%gives us more freedom in choosing a $\myb{\gamma}_{k}(n)$ that can maximize \no{any function of the model shift}. 
We propose a design based on the following observation.
Given $\mu$, a strongly convex loss function $F(\cdot)$ satisfies,
\begin{align}
    F(\myb{y}) \geq F(\myb{x}) + \nabla F(\myb{x})^{T}(\myb{y}-\myb{x}) + \frac{\mu}{2}||\myb{y}-\myb{x}||^{2}.
\end{align}
In the above, we set $\myb{x} = \myb{w}^{*}$, which is Bob's global optima and $\myb{y} =  \lim_{n \rightarrow \infty} \myb{w}^{Eve}(n)$. We see that the difference in Eve's and Bob's loss is lower bounded by the squared norm of the model shift. {However, due to the gradient term's dependence  on $\myb{x}$, maximizing the model shift may not always lead to an increase in Eve's loss.}
We will design functions of 
$\myb{\delta}_{k}(n)$ that $\myb{\gamma}_{k}(n)$ that yield the same FIM and condition for singularity as given in Proposition \ref{fimblock}, {while striving to maximize the model shift for Eve.}

\begin{prop}
    \label{h}
    We let $\myb{\gamma}_{k}(n) = h(\myb{\delta}_{k}(n))$, where $h(\cdots): \mathbb{R}^{d}\rightarrow\mathbb{R}^{d}$. The FIM conditions remain unchanged if $h(\cdots)$ satisfies
    \begin{align}
    \label{pde}
    \myb{\delta}_{k}(n)^{T}\nabla h_{k,n}(\myb{\delta}_{k}(n))  = 0.
\end{align}
\end{prop}
Proposition \ref{h} provides a condition on a family of functions of $\myb{\gamma}_{k}(n) = h(\myb{\delta}_{k}(n)$  for the FIM in proposition \ref{fimblock} to remain unchanged.

\section{MaxModShift}
\label{maxmod}
{While the shift strategies proposed in \cite{modshift} do introduce model shifts, the resulting magnitude of the shift is not controlled. In fact, when applying ModShift to the MNIST data set for a classification problem \cite{mnist}, we see that the gradient-driven shifts may not meaningfully change the loss for certain shift schemes. Herein, we devise strategies that enable stronger control of the model shift, subject to a power constraint at the agents.}

Let $N$ be the total number of rounds of learning. {We define the operator $\mathbf{H}(\cdot)\in {\cal R}^{K,N}$ such that $\mathbf{H}_{k,n}(\myb{\delta}_{k}(n)) = h_{k,n} (\myb{\delta}_{k}(n))$.  We propose the following optimization problem:}
%\begin{align}
 %   \label{objective}
  %  \arg \max_{\myb{\gamma}_{1}(1) \cdots \myb{\gamma}_{K}(N)} & \frac{1}{K^{2}}\left|\left|\sum_{k=1}^{K}\sum_{n=1}^{N}\myb{\gamma}_{k}(n)^{T}\myb{\delta}_{k}(n)\myb{u}\right|\right|^{2}\\
   % \label{c1}
    %\text{s.t. } \myb{\gamma}_{k}(n)^{T}\myb{u} & = -1 \hspace{5pt} \forall k,n,\\
    %\label{c2}
    %\left|\left|\myb{S}_{k}(n)\myb{\delta}_{k}(n)\right|\right|^{2} & \leq \alpha||\myb{\delta}_{k}(n)||^{2}, \forall k,n.
%\end{align}
\begin{align}
  \mathcal{P}: \arg \max_{\mathbf{H}(\cdot)} \frac{1}{m^{2}}\left(\sum_{k=1}^{K}\right. & \left.  \sum_{n=1}^{N} m_{k}h_{k,n}(\myb{\delta}_{k}(n))^{T}\myb{\delta}_{k}(n)\right)^{2}\left|\left|\myb{u}\right|\right|^{2} \nonumber
\end{align}
    \vspace*{-0.2in}
\begin{align}
    \label{c1}
    \text{s.t. } h_{k,n}(\myb{\delta}_{k}(n))^{T}\myb{u} & = -1 \hspace{5pt} \forall k,n\\
    \label{c2}
    \myb{\delta}_{k}(n)^{T}\nabla h_{k,n}(\myb{\delta}_{k}(n))  & = 0,  \forall k,n,\\
    \label{c3}
    \left|\left|\myb{S}_{k}(n)\myb{\delta}_{k}(n)\right|\right|^{2} \hspace{-2.5pt} + \hspace{-2.5pt} \left(h_{k,n}(\myb{\delta}_{k}(n))^{T}\myb{\delta}_{k}(n)\right)^{2}\hspace{-2.5pt} &\leq \alpha||\myb{\delta}_{k}(n)||^{2},\forall k,n.
\end{align}

From \cite{jsac}, we know that $\sigma^{2}_{max}(\myb{S}_{k}(n)) = \max\left(1,||\myb{u}_{k}(n)||^{2}||\myb{\gamma}_{k}(n)||^{2}\right)$. Thus we propose the following modified optimization problem, $ \mathcal{P}_{1}$:
\begin{align}
    \label{objective}
    \mathcal{P}_{1}: \arg \max_{\mathbf{H}(\cdot)} \frac{1}{m^{2}}\left(\sum_{k=1}^{K} \right. & \left.\sum_{n=1}^{N}m_{k}h_{k,n}(\myb{\delta}_{k}(n))^{T}\myb{\delta}_{k}(n)\right)^{2}\left|\left|\myb{u}\right|\right|^{2}\\
    \label{c4}
    \text{s.t. } h_{k,n}(\myb{\delta}_{k}(n))^{T}\myb{u} & = -1 \hspace{5pt} \forall k,n,\\
    \label{c5}
    \left|\left|h_{k,n}(\myb{\delta}_{k}(n))\right|\right|^{2} & = \frac{\alpha}{1+||\myb{u}||^{2}}, \forall k,n.
\end{align}
Note that the new power constraint is stronger than that of Equation (\ref{c3}) since,
\begin{align*}
    \left|\left|\myb{S}_{k}(n)\myb{\delta}_{k}(n)\right|\right|^{2} + (h_{k,n}(\myb{\delta}_{k}(n))^{T}\myb{\delta}_{k}(n))^{2} \leq\\ \sigma_{max}^{2}(\myb{S}_{k}(n))\left|\left|\myb{\delta}_{k}(n)\right|\right|^{2} + \left(||h_{k,n}(\myb{\delta}_{k}(n))||||\myb{\delta}_{k}(n)||\right)^{2} = \\\left|\left|h_{k,n}(\myb{\delta}_{k}(n))\right|\right|^{2}\left|\left|\myb{\delta}_{k}(n)\right|\right|^{2}\left(\left|\left|\myb{u}\right|\right|^{2}+1\right).
\end{align*}

In the current design, the shifts are not coupled across time, thus 
% Note that since the shifts are not coupled across time, i.e. the choice of $\myb{\gamma}_{k}(n)$ does not affect future updates since Bob compensates for the shifts,
the objective in $\mathcal{P}_{1}$ can be separated into $KN$ optimization problems where agent $k$ solves $N$ optimization problems to find $\mathbf{H}_{k,1}(\cdot) \cdots \mathbf{H}_{k,N}(\cdot).$ 
%In order to maximize the square of the sum of the shifts which is the objective in $\mathcal{P}_{1}$, either every shift $h_{k,n}(\myb{\delta}_{k}(n))^{T}\myb{\delta}_{k}(n)$ is maximized or every shift is minimized. 
To further simplify the optimization tasks, we pose two sub-optimal problems for which closed form solutions exist and which will yield improved shift designs.  For round $n$ and agent $k$, the optimization problems are,
\vspace*{-0.1in}

\begin{align}
    \label{obj1}
    \mathcal{P}^{max}_{k,n}:\arg \max_{\mathbf{H}(\cdot)}  & \; h_{k,n}(\myb{\delta}_{k}(n))^{T}\myb{\delta}_{k}(n)\\
    \text{s.t. } h_{k,n}(\myb{\delta}_{k}(n))^{T}\myb{u} & = -1,\\
    \left|\left|h_{k,n}(\myb{\delta}_{k}(n))\right|\right|^{2} & = \frac{\alpha}{1+\left|\left|\myb{u}\right|\right|^{2}}.
        \end{align}
      \begin{align}
    \label{obj2}
    \mathcal{P}^{min}_{k,n}:\arg \min_{\mathbf{H}(\cdot)}  & \; h_{k,n}(\myb{\delta}_{k}(n))^{T}\myb{\delta}_{k}(n)\\
    \text{s.t. } h_{k,n}(\myb{\delta}_{k}(n))^{T}\myb{u} & = -1,\\
    \left|\left|h_{k,n}(\myb{\delta}_{k}(n))\right|\right|^{2} & = \frac{\alpha}{1+\left|\left|\myb{u}\right|\right|^{2}}.
\end{align}
Note that we have dropped $m_{k}$ since $m_{k}\geq0 \forall k$ which implies that it will not change the expression obtained for $h_{k,n}(\myb{\delta}_{k}(n))$.

\begin{prop}
    \label{opt}
    The solutions to $\mathcal{P}^{max}_{k,n}$ and $\mathcal{P}^{min}_{k,n}$ are given by
    \begin{align}
         \myb{\gamma}_{k}^{\text{max}}(n) & \doteq h_{k,n}^{max}(\myb{\delta}_{k}(n)) = \chi_{k}(n) \myb{\delta}_{k}(n) -\\ &(1+(\myb{\delta}_{k}(n)^{T}\myb{u})\chi_{k}(n))\frac{\myb{u}}{\left|\left|\myb{u}\right|\right|^{2}},\\
         \myb{\gamma}_{k}^{\text{min}}(n) & \doteq h_{k,n}^{min}(\myb{\delta}_{k}(n)) = -\chi_{k}(n) \myb{\delta}_{k}(n) - \\&(1-(\myb{\delta}_{k}(n)^{T}\myb{u})\chi_{k}(n))\frac{\myb{u}}{\left|\left|\myb{u}\right|\right|^{2}}\\
\mbox{where} \;\;\; 
        \chi_{k}(n)& = \sqrt{\frac{\alpha'-1}{\left|\left|\myb{u}\right|\right|^{2}||\myb{\delta}_{k}(n)||^{2}-(\myb{\delta}_{k}(n)^{T}\myb{u})^{2}}} \;\;\; \mbox {and}\\
        & \alpha' = \alpha \frac{||\myb{u}||^{2}}{1+||\myb{u}||^{2}}.   
    \end{align}
    We require $\alpha' \geq 1$ for $\chi_{k}(n)$ to be real.
    %The corresponding minimum and maximum values of the objectives in $\mathcal{P}_{k,n}^{max}$ and $\mathcal{P}_{k,n}^{min}$ are  
    % \begin{align}
    %     (\myb{\gamma}_{k}^{\text{max}}(n))^{T}\myb{\delta}_{k}(n) & = -\frac{(\myb{\delta}_{k}(n)^{T}\myb{u})}{||\myb{u}||^{2}} + \frac{\alpha'-1}{\chi_{k}(n)||\myb{u}||^{2}},\\
    %     (\myb{\gamma}_{k}^{\text{min}}(n))^{T}\myb{\delta}_{k}(n) & = -\frac{(\myb{\delta}_{k}(n)^{T}\myb{u})}{||\myb{u}||^{2}} - \frac{\alpha'-1}{\chi_{k}(n)||\myb{u}||^{2}},
    % \end{align}
    Furthermore, both $\myb{\gamma}_{k}^{\text{max}}(n)$ and $\myb{\gamma}_{k}^{\text{min}}(n)$ satisfy the condition given in Equation \ref{pde}.
\end{prop}

%\begin{proof}
 %   Refer Appendix \ref{optproof}
%\end{proof}

\emph{Proof Sketch:} We use the method of Lagrange multipliers to solve the optimization problems $\mathcal{P}_{k,n}^{\text{max}}$ and $\mathcal{P}_{k,n}^{\text{min}}$. We note that $\alpha \geq 1$ and from the Cauchy-Schwarz inequality, $\left|\left|\myb{u}\right|\right|^{2}||\myb{\delta}_{k}(n)||^{2}-(\myb{\delta}_{k}(n)^{T}\myb{u})^{2} \geq 0$. Algebraic manipulations yield the condition for the shift to be zero:
%We now check if this shift can be $0$. For the shift to be $0$,
\begin{align}
%    (\myb{\delta}_{k}(n)^{T}\myb{u})^{2} & = \left(\alpha - 1 \right)\left(\left|\left|\myb{u}\right|\right|^{2}||\myb{\delta}_{k}(n)||^{2}-(\myb{\delta}_{k}(n)^{T}\myb{u})^{2}\right),\\
    \frac{(\myb{\delta}_{k}(n)^{T}\myb{u})^{2}}{\left|\left|\myb{u}^{2}\right|\right|||\myb{\delta}_{k}(n)||^{2}} & = 1-\frac{1}{\alpha'}.
\end{align}
We may increase or decrease $\alpha$ to ensure that the above does not occur. 

Since the shifts in each round depend on the update, $\myb{\delta}_{k}(n)$ and $\myb{u}$, we do not know \emph{a priori} whether the solution from ${\cal P}^{min}$ or ${\cal P}^{max}$ 
yields the larger magnitude shift, as this is governed by the angle between $\myb{\delta}_{k}(n)$ and $\myb{u}$.  To this end, we consider one or the other. Thus either all agents use $\myb{\gamma}_{k}^{\text{max}}(n) \hspace{5pt} \forall n$ or use $\myb{\gamma}_{k}^{\text{min}}(n) \hspace{5pt} \forall n$.

\section{Simulation Results}
\label{results}
We consider a linear regression problem on a synthetic dataset in a federated learning setup with $100$ independent agents communicating with Bob. The data vectors $\myb{x}_{k,i}$ are of dimension $d = 60$ and are Gaussian distributed with mean $0$ and identity covariance matrix. The weight vector $\myb{w} = [1,2, \cdots ,d]^{T}$ is fixed and labels for each data vector are generated as 
$y_{k,i} = \myb{w}^{T}\myb{x}_{k,i} + n_{k,i}$
where $n_{k,i}$ is a zero-mean Gaussian random variable with standard deviation $0.1$ for all $k,i$. Each agent has a dataset of size $1,000$. We use the mean square error loss function for which $l(\boldsymbol{w},\boldsymbol{x}_{k,i},y_{k,i})$ in Equation \eqref{loss} is given by 
\begin{align}
    l(\boldsymbol{w},\boldsymbol{x}_{k,i},y_{k,i}) = (\myb{w}(n)^{T}\myb{x}_{k,i}-y_{k,i})^{2}.
\end{align}
We set $\myb{u} = \mathbf{1}, \frac{\sigma_{B}^{2}}{(h_{k}^{B}(n))^{2}} = \frac{\sigma_{E}^{2}}{(h_{k}^{E}(n))^{2}} = 0.1 \hspace{5pt} \forall k,n$ and $\myb{w}^{Bob}(0) = \myb{w}^{Eve}(0) = \myb{0}$ where $\myb{0} \in \mathbb{R}^{d}$ is the zero vector. The learning rate is $\eta=0.005$ and the number of local gradient rounds is $R = 10$. In order to compare MaxModShift with other schemes, we define
\begin{align}
    P_{avg} = \frac{\sum_{n=1}^{N}\sum_{k=1}^{K}\left(\left|\left|\myb{\delta}_{k}^{E}(n)\right|\right|^{2} + s_{k,n}^{2}\right)}{\sum_{n=1}^{N}\sum_{k=1}^{K}\left|\left|\myb{\delta}_{k}(n)\right|\right|^{2}},
\end{align}
where $\myb{\delta}_{k}^{E}(n)$ is the shifted difference and $s_{k,n}$ is the scalar shared through the secret channel. Note that $P_{avg} \leq \alpha$ for MaxModShift and is a less constrained measure of power when compared to the power constraint in Equation (\ref{c3}). 

We compare MaxModShift to a noise injection scheme where each agent adds a noise vector to the differences $\myb{\delta}_{k}(n)$.  {The noise injection scheme is inspired by differential privacy \cite{dp}. While we do not formally enforce differential privacy since we do not perform gradient clipping, the added noise serves to obscure individual updates.} We consider Gaussian noise with covariance matrix $\beta^{2}\myb{I}$. The $d$ elements of the noise vector are shared with Bob via the secret channel.

Figure \ref{fig:alpha} shows Eve's loss after the learning is completed and Figure \ref{fig:power} plots $P_{avg}$ for various schemes. We first compare the cases where all agents use $\myb{\gamma}_{k}^{\text{max}}(n)$ (labelled 'MaxModShift (Max)') and all agents use $\myb{\gamma}_{k}^{\text{min}}(n)$ (labelled 'MaxModShift (Min)') as a function of $\alpha$. We observe that choosing to minimize each shift gives a larger final loss for each $\alpha$ after learning is completed. Both these schemes are close in $P_{avg}$ with the difference being due to the power needed to transmit the scalar over the secret channel which is larger for the MaxModshift (Min) scheme. We also compare the above schemes with the Max scheme which is a ModShift shift scheme. Since it is not a function of $\alpha$ or $\beta$, it is parallel to the horizontal axis in both figures. We observe that MaxModShift (Min) leads to a higher final loss for Eve when compared to the Max scheme for all $\alpha$. While the Max scheme does lead to a higher loss for smaller $\alpha$ when compared to MaxModShift (Max) , it is dominated by MaxModShift (Max) for $\alpha \geq 1.2.$ Furthermore, the Max scheme require $5$ times more transmission power when compared to the case where no shifts are being added. Thus, both MaxModShift schemes lead to better performance with just $24\%$ of the transmission power required by ModShift. This demonstrates how MaxModShift is able to lead to larger model shifts while requiring lesser power when compared to ModShift.
\begin{figure}[t]
\begin{center}
\includegraphics[width=0.5\textwidth,height=2.5in]{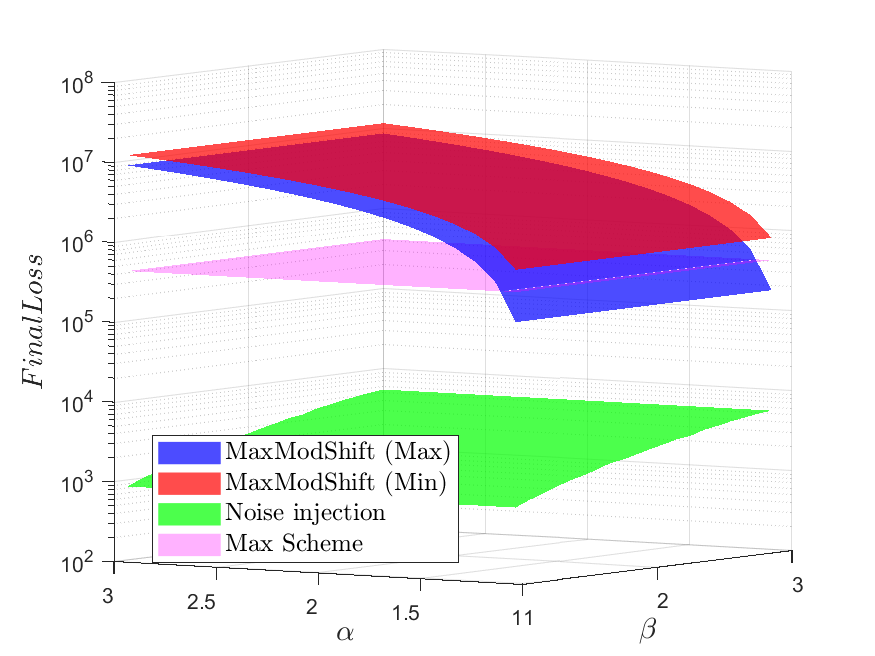}
\caption{Eve's final loss as a function of $\alpha$ and $\beta$}
\label{fig:alpha}
\vspace*{-0.2in}
\end{center}
\end{figure}

\begin{figure}[t]
\begin{center}
\includegraphics[width=0.5\textwidth,height=2.5in]{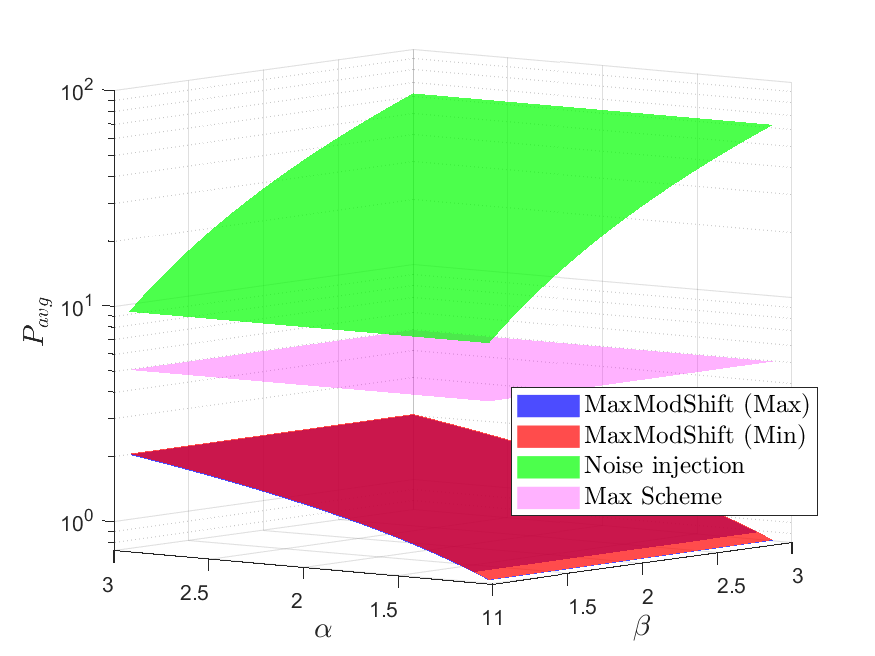}
\caption{$P_{avg}$ as a function of $\alpha$ and $\beta$}
\label{fig:power}
\end{center}
\vspace*{-0.2in}
\end{figure}

  Next, we compare MaxModShift with the noise injection scheme. We bias our analysis of the required average transmission power towards the noise injection scheme by assuming $s_{k,n} = 0 \hspace{5pt} \forall k,n$ for the noise injection scheme and show that MaxModShift offers better performance despite this. From figure \ref{fig:alpha}, we note that both MaxModShift schemes consistently lead Eve to a higher final loss than than the noise injection schemes. Moreover, the noise injection scheme requires much larger $P_{avg}$ and is at least $4$ times the highest $P_{avg}$ of the MaxModShift schemes.   $P_{avg}$ attains high values under noise injection largely due to $\left|\left|\myb{\delta}_{k}(n)\right|\right|$ getting smaller as $n$ increases while the noise variance does not scale down which implies that $\left|\left|\myb{\delta}_{k}^{E}(n)\right|\right|$ is eventually dominated by the noise term. Thus, the numerator of $P_{avg}$ does not scale down as the denominator reduces which leads to the large $P_{avg}$ values. In the case of MaxModShift, the power constraint in Equation  (\ref{c3}) ensures that the power required for transmission scales down with $\left|\left|\myb{\delta}_{k}^{E}(n)\right|\right|$. While $\beta$ may be posed as a function of time to reduce $P_{avg}$, we note that the final loss under the noise injection scheme is already roughly $4$ orders of magnitude lower than our scheme. This, added to the existing problem of the noise addition scheme requiring more secret channel usage, makes the noise addition scheme a less attractive shifting strategy.  

\section{Conclusions}
{In this paper, we have strongly generalized our prior designs \cite{modshift} for injecting model shifts into a federated learning problem for an eavesdropper. Herein, we generalize our shift design} and recompute the Fisher Information Matrix for the eavesdropper's learning problem as well as the conditions for driving this matrix to singularity.  {A new shift design which endeavors to maximize the model shift (and thus the eavesdropper's loss) is proposed and is subject to a power constraint. } The resulting design requires the selection between one of two shift parameters.  As with \cite{modshift}, the new designs pass a model tamper test that the eavesdropper may employ, thus rendering the use of the model shifts more secure. The new schemes are shown to be provably good and the overall algorithm converges.  In numerical comparisons against noise injection and ModShift, MaxModShift offers superior performance in hiding the model from an eavesdropper while requiring significantly lower transmission power. While our current solution makes a single choice before the learning starts, a greedy approach where agents must make decisions in each round based on feedback from Bob may be used to address the decision making problem. This requires characterizing the feedback shared by Bob and is left for future work.

\bibliographystyle{IEEEtran}
\bibliography{References/references}

\end{document}